\documentclass[conference]{IEEEtran}
\usepackage{savesym}
\usepackage{amsmath,amssymb,amsfonts,amsthm}
\usepackage[numbers]{natbib}
\usepackage{comment}
\usepackage{xspace}
\usepackage{scalerel}
\usepackage{xcolor}
\usepackage{multicol}
\usepackage[bookmarks=true]{hyperref}
\usepackage{graphicx,subfigure}
\usepackage{authblk}

\usepackage{times}  
\usepackage{helvet}  
\usepackage{courier}  
\usepackage{url}  
\usepackage{graphicx}  
\usepackage{mathrsfs}
\DeclareMathAlphabet{\mathpzc}{OT1}{pzc}{m}{it}
\usepackage{xcolor}

\usepackage{amsthm,amssymb,amsmath,amsthm}
\usepackage{xspace}
\usepackage{tikz}
\usetikzlibrary{arrows, shapes,positioning,decorations.pathreplacing, decorations.pathmorphing}
\usepackage[ruled]{algorithm2e}
\usepackage{flushend}

\mathchardef\dash="2D
\DeclareMathOperator*{\E}{\mathbb{E}}
\DeclareMathOperator*{\argmin}{\arg\min}
\newcommand{\st}{s\dash t}
\usepackage{color-edits}
\addauthor{nh}{magenta}
\addauthor{ss}{blue}
\newtheorem{theorem}{Theorem}

\newtheorem{claim}{Claim}

\theoremstyle{definition}
\newtheorem{definition}{Definition}

\newenvironment{claimproof}
 {\proof}
{\endproof}

\allowdisplaybreaks

\newcommand{\calE}{\ensuremath{\mathcal{E}}\xspace}
\newcommand{\calG}{\ensuremath{G}\xspace}
\newcommand{\calR}{\ensuremath{R}\xspace}

\newcommand{\calX}{\ensuremath{X}\xspace}

\newcommand{\calO}{\ensuremath{O}\xspace}
\newcommand{\calP}{\ensuremath{\mathcal{P}}\xspace}

\newcommand{\calW}{W}
\renewcommand{\P}{\calP}

\newcommand{\R}{\mathbb{R}}

\newcommand{\cost}{\mathrm{cost}}
\newcommand{\alg}{\textsc{Alg}\xspace}
\newcommand{\Lazy}{\textsc{LazySP}\xspace}
\renewcommand{\O}{\mathcal{O}}
\newcommand{\Lazies}{\mathpzc{LazySP}}

\newcommand{\PSPACE}{{\small \ensuremath{\mathsf{PSPACE}}\xspace}}

\newcommand\blfootnote[1]{%
  \begingroup
  \renewcommand\thefootnote{}\footnote{#1}%
  \addtocounter{footnote}{-1}%
  \endgroup
}

\begin{document}

\title{	The Provable Virtue of Laziness in Motion Planning
}


\author[1]{ Nika Haghtalab}
\author[1]{ Simon Mackenzie}
\author[1]{ Ariel D. Procaccia}
\author[2]{ Oren Salzman}
\author[3]{ Siddhartha S. Srinivasa}

\affil[1]{\small Computer Science Department, Carnegie Mellon University}
\affil[2]{\small Robotics Institute, Carnegie Mellon University}
\affil[3]{\small School of Computer Science \& Engineering, University of Washington}

\maketitle
\thispagestyle{empty}
\pagestyle{empty}


\begin{abstract}
The \emph{Lazy Shortest Path (LazySP)} class consists of motion-planning algorithms that only evaluate edges along shortest paths between the source and target. These algorithms were designed to minimize the number of edge evaluations in settings where edge evaluation dominates the running time of the algorithm; but how close to optimal are LazySP algorithms in terms of this objective? Our main result is an analytical upper bound, in a probabilistic model, on the number of edge evaluations required by LazySP algorithms; a matching lower bound shows that these algorithms are asymptotically optimal in the worst case.\blfootnote{
\llap{\textsuperscript{*}}
This work was partially supported by
a Microsoft Research Ph.D. fellowship,
National Science Foundation 
IIS-1409003, IIS-1350598, IIS-1714140, CCF-1525932, and CCF-1733556, 
Office of Naval Research N00014-16-1-3075 and N00014-17-1-2428 
and by a
Sloan Research Fellowship.
} 
\end{abstract}

\section{Introduction}

The simplest motion planning model~\cite{L06, S04} involves a robot system~$\calR$ moving in a workspace $\calW \in \{ \R^2, \R^3 \}$ cluttered with obstacles~$\calO$. 
Given an initial placement $s$ and a target placement $t$ of~$\calR$, we wish to determine whether there exists a collision-free motion of~$\calR$ connecting~$s$ and~$t$, and, if so, to plan such a motion.

Typically,~$\calR$ is abstracted as a point, or a \emph{configuration}, in a high-dimensional space called the \emph{configuration space}~\calX, where each configuration maps~$\calR$ to a specific placement in~$\calW$~\cite{L83}.
The configuration space is subdivided into the free and forbidden spaces, corresponding to placements of~\calR that are free or that intersect with an obstacle, respectively.
Since the general motion-planning problem is
\PSPACE-hard~\cite{HSS84},  
a common approach is to use sampling-based algorithms~\cite{HLM99,KF11,KSLO96,LK99}.
These algorithms approximate~$\calX$ via a discrete graph~$\calG$ called a \emph{roadmap}. 
Vertices in~$\calG$ correspond to sampled configurations in~$\calX$, and edges in~$\calG$ correspond to local paths (typically straight lines).
Approximately solving the motion-planning problem thus reduces to the problem of finding a collision-free shortest path in~$\calG$ between the vertices corresponding to $s$ and $t$.

Testing if a vertex or an edge of~$\calG$ is collision free requires one or more geometric tests called \emph{collision detection}.
Arguably, collision detection in general, and edge evaluation in particular, are the most time-consuming operations in sampling-based algorithms~\cite{CBHKKLT05,L06}.
Thus, path planning on~\calG differs from traditional search algorithms such as Dijkstra~\cite{D59} or A*~\cite{HNR68},  where the graph is typically implicit and large, but edge evaluation is trivial compared to search.
Indeed, much recent work in motion planning focuses on evaluating the edges of~$\calG$ \emph{lazily}, informed by the search algorithm as it progresses~\cite{BK00,CSCS17,DS16,K15,SH15}.

In a recent paper, Dellin and Srinivasa~\cite{DS16} present a unifying formalism for shortest-path problems where edge evaluation dominates the running time of the algorithm.
Specifically, they define and investigate a class of algorithms termed \emph{Lazy Shortest Path (LazySP)}, which run any shortest-path algorithm on~\calG followed by evaluating the edges along that shortest path.
The algorithms are differentiated by an \emph{edge selector} function, which chooses the 
 edges the algorithm evaluates along the shortest path.
Dellin and Srinivasa show that several prominent motion-planning algorithms are captured by LazySP, using a suitable choice of this selector.
Furthermore, they  extensively evaluate the algorithm  empirically on a wide range of edge selectors. 
Their experiments range from 
toy scenarios, which demonstrate the advantages of each edge selector, to articulated 7D motion-planning problems that show that, using this approach, nontrivial problems can be solved within seconds.

LazySP was proposed as an algorithm that attempts to minimize the overall number of edges evaluated (or \emph{queried}) in the process of solving a given motion-planning problem. 
A natural question to ask is
\begin{quote}
\emph{... what is the query complexity of LazySP, and is its query complexity the best possible?}
\end{quote}
In other words, can we bound the number of edges evaluated by LazySP as a function of the complexity of the roadmap~\calG? And are there algorithms not in this class that have lower query complexity?

To address these questions, we need to explicitly model how queries are answered. 
We start in Section~\ref{sec:det} by considering the \emph{deterministic} setting, where the set of collision-free edges is determined upfront. 
Our first result establishes that, in this model, it is optimal to always test edges along the shortest path, i.e., in every instance there is an edge selector for which LazySP is optimal. Although the edge selector in question requires full access to the set of collision-free edges, so the real-world implications of this result are limited, it does provide a theoretical underpinning for the idea of restricting queries to shortest paths, which lies at the heart of LazySP. 

In practice, we are interested in a slightly more complex model, which we call the \emph{probabilistic} setting; it is explored in Section~\ref{sec:rand}.
Here, each edge is endowed with a probability of being in collision---a common assumption in motion planning (see, e.g.,~\cite{CDS16})---and we are interested in policies that minimize the query complexity, that is, policies that minimize the \emph{expected} number of steps until the algorithm finds the shortest path or declares that no path exists.
We first show that there are instances where LazySP is suboptimal, regardless of the edge selector. In a nutshell, we describe a delicate construction where initially querying edges that are not on the shortest path provides valuable information for subsequent queries. 

So, in the probabilistic setting, LazySP is just a proxy for the (presumably intractable) optimal policy, but is it a good proxy? We answer this question in the positive. Our main result is that the query complexity of LazySP (with an edge selector satisfying a certain \emph{connectivity} property) is bounded by $ O(n/p)$ edge evaluations with high probability, where 
$n$ is the number of vertices in~\calG,
and
$p$ is the minimum probability on any edge.
We complement this result with an $\Omega(n/p)$ lower bound that holds for \emph{every} algorithm that is guaranteed to be correct. We conclude that, from a worst-case viewpoint, LazySP is, in fact, (asymptotically) optimal.

\section{The Model}
\label{sec:model}

An instance of our problem is given by a multigraph $G = (V, E)$ --- that is, there may be multiple edges between two vertices --- whose set of vertices includes two distinguished vertices: the source vertex $s$ and the target vertex $t$. We deal with multigraphs, rather than simple graphs, mostly for ease of exposition; see Section~\ref{sec:disc} for a discussion of this point. We simply refer to $G$ as a \emph{graph} hereinafter. 

We say that a graph $G' = (V, E')$ is a subgraph of $G$ if $E'\subseteq E$. Given a graph $G = (V, E)$ and its subgraph $G' = (V, E')$, an oracle $\O_{G'}^G$ is a function that takes as input an edge $e\in E$ and returns \textsc{Yes} if $e\in E'$, and \textsc{No} otherwise. When $G$ is clear from the context, we suppress it in this notation.

In the \emph{path-finding} problem, an algorithm $\alg$ is given a graph $G$ and an oracle $\O_{G'}$. The goal of the algorithm is to find the shortest $\st$ path in $G'$. Since $G'$ is not revealed to the algorithm directly, the algorithm has to query $\O_{G'}$ on specific edges of $G$ to find a path. 
That is, $\alg(G, \O_{G'})$ issues a sequence of edge queries to $\O_{G'}$, and upon termination, returns an $\st$ path or decides that none exists.

For an algorithm to be \emph{correct}, we require that it correctly identify a shortest $\st$ path in $G'$, or that it \emph{certify} that none exists (by invalidating every possible path), for any $G$ and $G'\subseteq G$. 
Therefore, a correct algorithm can only terminate when the solution it provides continues to be correct even if the responses to unqueried edges are selected adversarially.
More formally, let $Q\subseteq E$ be the set of edges queried by a correct algorithm $\alg$ on $G$ and $\O_{G'}$. Let $Q_y = Q \cap E'$ and $Q_n = Q \setminus E'$ be the set of queried edges that, respectively,  belong and do not belong to $G'$. Then $\alg$ can terminate only if there is a shortest $\st$ path in $G'$, denoted $P^*$, such that $P^*\subseteq Q_y$, and there is no $\st$ path in $(V,E \setminus Q_n)$ that is shorter than $P^*$. If no path exists, then $\alg$ can terminate only if there is no $\st$ path in $(V,E \setminus Q_n)$.

Clearly, an algorithm that first queries all edges in $E$, thereby fully constructing $G'$, and only then finds the shortest $\st$ path, is a correct algorithm. However, such an algorithm may use a large number of queries, some of which may be unnecessary. In this paper, we are interested in  algorithms  that find a shortest $\st$ path using a minimal number of queries.
We denote the number of queries that $\alg$ makes on input $G$ and $\O_{G'}$ by $\cost(\alg(G, \O_{G'}))$.

We are especially interested in the $\Lazies$ class of algorithms, introduced by Dellin and Srinivasa~\cite{DS16}. Any algorithm in the class $\Lazies$ is determined by an \emph{edge selector}, which, informally, decides which edge to query on a given $\st$ path. Formally, let $\P$ be the set of all $\st$ paths in $G$. An edge selector is a function $f: \P \times 2^E\times 2^E \rightarrow E$ that takes any $\st$ path $P\in \P$, a subset of queried edges $Q_y$ that are in $E'$, and a subset of queried edges $Q_n$ that are not in $E'$, and returns an edge $e\in P\setminus Q$. Examples of edge selectors include:
\begin{itemize}
\item \emph{Forward edge selector}: Returns the first unqueried edge in $P$, that is, the one closest to $s$.  

\item \emph{Backward edge selector}: Returns the last unqueried edge in $P$, that is, the one closest to $t$. 

\item\emph{Bisection edge selector}: 
Returns an unqueried edge in $P$ which is furthest from an evaluated edge on the path. 
\end{itemize}

Given an edge selector $f$, the corresponding $\Lazy_f\in \Lazies$ is described in Algorithm~\ref{alg:LazySP}.
At a high level, $\Lazy_f$, in a given time step, considers a candidate \emph{shortest} $\st$ path $P$ over all those edges whose existence has not yet been ruled out by the oracle. Then, it uses the edge selector to query an unqueried edge $e\in P$. It updates the set of queried edges and repeats.
At any point, if the edges of path $P$ that is currently under consideration are all verified, the algorithm terminates and returns $P$. If no viable $\st$ paths remain, the algorithm terminates and certifies that no $\st$ path exists in $G'$.

\begin{algorithm}[ht]
\textbf{input:} Graph $G$ and oracle $\O_{G'}$

\medskip

$Q_n \gets \emptyset$
\tcc*{ \footnotesize {in-collision evaluated edges }}

$Q_y \gets \emptyset$
\tcc*{ \footnotesize {collision-free evaluated edges}}

\While{there exists\footnotemark{} a shortest $\st$ path $P$ in $E\setminus Q_n$} {
    \lIf{$P\subseteq Q_y$}{\Return $P$}
    
    $e\gets f(P,Q_y,Q_n)$
    \tcc*{ \footnotesize {select edge along $P$}}
    
    \lIf{$\O_{G'}(e) = \textsc{Yes}$}{$Q_y \gets Q_y \cup \{e\}$}
    \lElse{$Q_n \gets Q_n \cup \{e\}$}
}

\Return{$\emptyset$};
\caption{\textsc{$\Lazy_f$}}
\label{alg:LazySP}
\end{algorithm}
\footnotetext{If there are multiple  $\st$ paths of the same length, the algorithm breaks ties according to a consistent tie-breaking rule.}

It is not hard to see that any algorithm in the class $\Lazies$ is a correct algorithm. This is due to the fact that these algorithms always consider the shortest path that has not yet been ruled out. Therefore, upon termination, they return the shortest $\st$ path in $G'$. Moreover, an edge selector never returns an edge that has been queried before and, hence, these algorithms never query an edge more than once. It follows that any such algorithm eventually terminates. See the paper of Dellin and Srinivasa~\cite{DS16} for a more detailed discussion of the $\Lazies$ class.

Let us conclude this section with an example of the execution of \Lazy with the forward edge selector, which also illustrates some of the terminology introduced earlier. Figure~\ref{fig:example} shows the set of vertices $V=\{s,t,a,b,c,d,e\}$ shared by $G$ and $G'$, as well as two types of edges: those in $E'$, shown as solid edges, and those in $E\setminus E'$, shown as dashed edges. The order in which edges are queried is shown as labels on the edges. This order on edge queries is induced by evaluating shortest paths in the following order: $sat$, $sabt$, $scdt$, $sabdt$, and $sabdet$. 

\begin{figure}[h]
\begin{center}
\begin{tikzpicture}
	\usetikzlibrary{decorations.pathreplacing}
    \tikzstyle{circ}=[draw,circle,fill=white,minimum size=9pt,
                            inner sep=0pt]

    \draw (0,0) node[circ] (s) {\small $s$};
	\draw (1.5,-1) node[circ] (c) {\small $c$};
	\draw (1.5,0) node[circ] (b) {\small $b$};
	\draw (1.5,1) node[circ] (a) {\small $a$};
	\draw (3,-1) node[circ] (d) {\small $d$};
	\draw (4.5,-1) node[circ] (e) {\small $e$};
	\draw (6,0) node[circ] (t) {\small $t$};

	\draw (s) -- node [above] {\tiny $1$} (a);
	\draw [dashed] (s) -- node [below] {\tiny $5$} (c);
	\draw [dashed] (a) -- node [above=-1pt] {\tiny $2$} (t);
	\draw (a) -- node [right=-1pt] {\tiny $3$} (b);
	\draw [dashed] (b) -- node [above=-1pt] {\tiny $4$} (t);
	\draw (c) -- (b);
	\draw (c) -- (d);
	\draw [dashed] (d) -- node [above=-1pt] {\tiny $7$} (t);
	\draw (d) -- node [below=-1pt] {\tiny $8$} (e);
	\draw (e) -- node [below] {\tiny $9$} (t);
	\draw (b) -- node [above] {\tiny $6$} (d);
\end{tikzpicture}
\caption{Example of the execution of \Lazy with the forward edge selector. Solid edges are in $E'$, dashed edges are in $E\setminus E'$.}
\label{fig:example}
\end{center}
\vspace{-5mm}

\end{figure}
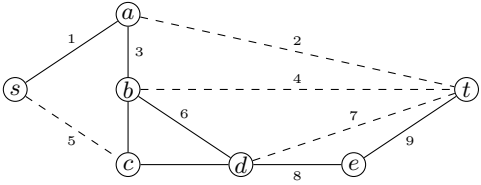


\section{The Deterministic Setting}
\label{sec:det}

In this section, we consider the problem of using a minimum number of edge queries to find a shortest $\st$ path, or verifying that no $\st$ path exists, when a subgraph $G'\subseteq G$ is \emph{deterministically} chosen (but not revealed to the algorithm).

In more detail, let $G = (V, E)$ be a graph, and let $G' = (V, E')$ be its subgraph. Recall that $\cost(\alg(G, \O_{G'}))$ denotes the number of edge queries $\alg$ makes on graph $G$ when oracle responses are according to graph $G'$. 
Our first result asserts that the class $\Lazies$ is optimal in this setting, in the sense that for any correct algorithm there is a \Lazy algorithm (with a specific edge selector) that finds the shortest path using at most as many queries.

\begin{theorem}
\label{thm:det}
For any graph $G$ and $G'\subseteq G$, and any correct algorithm $\alg$, there exists $\alg' \in \Lazies$ such that 
\[ \cost(\alg'(G, \O_{G'})) \leq \cost(\alg(G, \O_{G'})).
\]
\end{theorem}
\begin{proof}
Let $P^*$ be a shortest $\st$ path in $G'$ and let $P_1, \dots, P_m$ be the $\st$ paths in $G$ that are strictly shorter than $P^*$, ordered by their length. If no path in $G'$ exists, $P_1, \dots, P_m$ is the list of all $\st$ paths in $G$. Let $Q$ be the set of edges queried by $\alg(G, \O_{G'})$, $Q_y = Q \cap E'$, and $Q_n = Q\setminus E'$. 

Let $Q^*$ be the set that includes all edges of $P^*$, all of which exist in $G'$, and an optimal  cover for the sets $P_i$ using edges that do not belong to $G'$. That is, let $Q^* =  P^* \cup Q^*_n$, where
\[ Q^*_n = \argmin_{S\subseteq E \setminus E'} \{ |S|:\ \forall i\in[m], P_i \cap S \neq \emptyset \}.
\]
We argue that $|Q| \geq |Q^*|$. This is due to the fact that  correctness of $\alg$ implies that $P^* \subseteq Q_y$, and any path that is shorter than $P^*$ has an invalidated edge, i.e., for all $i\in[m]$, $P_i \cap Q_n \neq \emptyset$. Note that the latter condition shows that $Q_n$ is a cover for the sets $P_i$ using edges $E\setminus E'$, so by the optimality of $Q^*_n$, we have
\[ |Q| = |Q_y| + |Q_n| \geq |P^*| + |Q^*_n| = |Q^*|.
\]
It remains to show that there is  $\alg' \in \Lazies$ that only queries edges in $Q^*$. 
Let $\alg'$ be the algorithm that first queries an edge in $Q^*_n\cap P_1$ (there must be one), then an edge in $Q^*_n\cap P_2$ (if it is nonempty), and so on, until $Q^*_n\cap P_m$ (if it is nonempty), and finally queries all the edges in $P^*$. We argue that $\alg'$ must query all the edges in $Q^*$. Indeed, the only difficulty is that, in principle, it may be the case that at some point $P_1,\ldots,P_k$ have already been invalidated, and there is some $e\in Q^*_n$ such that $e\notin P_{k+1}\cup\cdots\cup P_m$, meaning that $e$ cannot be queried in the future. But, in that case, $e$ is not needed in order to invalidate the paths $P_1,\ldots,P_m$, in contradiction to the optimality of $Q^*$ (and that of $Q^*_n$, specifically). 

Note that $\alg'$ has the property that at any time it only queries edges on the shortest $\st$ path that has not been invalidated yet. Clearly, it is possible to define an edge selector that makes the same choices as $\alg'$. We conclude that $\alg'$, whose cost is at most that of $\alg$, can be represented as a member of $\Lazies$. 
\end{proof}

We can alternatively interpret Theorem~\ref{thm:det} in a model where \Lazy may be equipped with an \emph{omniscient} edge selector that has full access to $G'$. In particular, this omniscient edge selector can compute $Q^*$, which, by the way, requires solving an $\mathsf{NP}$-hard variant of \textsc{Set Cover}. Even though the algorithm already knows $G'$, it still has to issue queries as it must \emph{certify} that $P^*$ is indeed the shortest path (if an $\st$ path exists). 

Clearly, an omniscient edge selector is impractical. The significance of Theorem~\ref{thm:det}, therefore, is mostly conceptual. It suggests that the restriction that algorithms must always query edges on the current shortest path is not a barrier to optimality. This gives theoretical justification for the $\Lazies$ class. However, as we shall see shortly, the message is more nuanced when the outcomes of queries are randomized.

\section{The Probabilistic Setting}
\label{sec:rand}

In this section, we consider a probabilistic variant of the setting we investigated in Section~\ref{sec:det}. We view the probabilistic model as a closer fit with reality than its deterministic counterpart.
 
In more detail, let $p\in (0,1)$ be the probability that any given edge in $G$ exists in $G'$. In a more general setting with different probabilities associated with different edges, we can simply think of $p$ as a lower bound on the probabilities for query upper bounds, or as an upper bound on the probabilities for query lower bounds. We denote by $G'\sim_p G$ the process of generating a random graph $G' = (V, E')$ from~$G$ by allowing each $e\in E$ to belong to $E'$ with probability~$p$, independently. We suppress $p$ in this notation when it is clear from the context.

In the current setting, a subgraph $G' = (V, E')\sim_p G$ is realized according to edge probability $p$, but it is not revealed to the algorithm. As before, the algorithm receives $G$ and $\O_{G'}$ as input, and uses $\O_{G'}$ to verify whether an edge exists. The goal of the algorithm is to minimize the \emph{expected} number of edge queries over $G'\sim_p G$, such that it \emph{correctly} either
\begin{enumerate}
\item returns a path that is the shortest $\st$ path in $G'$, or
\item certifies that there is no $\st$ path in $G'$.
\end{enumerate}
Note that, although the expected number of queries an algorithm issues is taken over $G'\sim G$, the correctness condition must hold for \emph{every} $G'$.

\subsection{Suboptimality of LazySP}

Our next result asserts that the class of algorithms $\Lazies$ does not always include an optimal query policy, which minimizes the expected number of queries. 
At a high level, the reason behind this is that, in some graphs, querying a few edges that are not on the shortest path can identify the most important regions of the graph, which should be explored next. 
To see this, consider the graph in Figure~\ref{fig:full}. In this graph, the arcs marked by $A$ and $B$ each include multi-edge structures shown in Figures~\ref{fig:thick} and \ref{fig:string}, respectively. Structures~$A$ and $B$ are designed so that arcs labeled by $B$ are much longer than $A$, so any \Lazy algorithm starts by querying the arcs labeled by $A$.

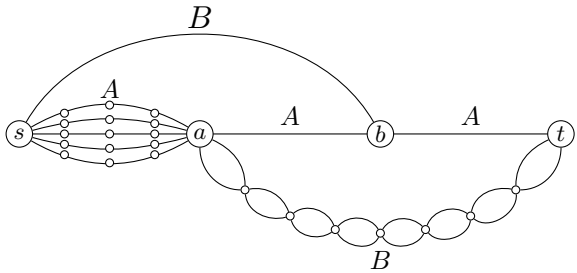
\begin{figure}[b]
	\begin{center}
\begin{tikzpicture}[scale = 1.2]
	\usetikzlibrary{decorations.pathreplacing}
    \tikzstyle{every node}=[draw,circle,fill=white,minimum size=10pt,
                            inner sep=0pt]

    \draw (0,0) node (A) {{\small $s$}}
    -- ++ (2cm, 0) node (B) {{\small $a$}} 
    -- ++ (2cm, 0) node (C){{\small $b$}}
    -- ++ (2cm, 0) node (D) {{\small $t$}};

    \draw (1cm, 0.5cm) node[draw = none, fill = none]{$A$};
    \draw (3cm, 0.2cm) node[draw = none, fill = none]{$A$};
    \draw (5cm, 0.2cm) node[draw = none, fill = none]{$A$};
    
    \draw (A) to [out=60,in=120] (C) node[midway, draw = none, fill = none, xshift = 2cm, yshift = 1.3cm]{$B$};    

    \draw (A) to [out=30,in=150] (B){};
	\draw (A) to [out=15,in=165] (B){};
    \draw (A) to [out=-30,in=-150] (B){};
	\draw (A) to [out=-15,in=-165] (B){};

	\draw (04cm, -1.4cm) node[draw = none, fill = none]{$B$};
	
	\draw (0.5cm, 0cm) node[draw, fill =white,minimum size= 3pt ]{};
	\draw (1cm, 0cm) node[draw, fill =white,minimum size= 3pt ]{};
	\draw (1.5cm, 0cm) node[draw, fill =white,minimum size= 3pt ]{};

	\draw (1cm, 0.32cm) node[draw, fill =white,minimum size= 3pt ]{};
 	\draw (1cm, -0.32cm) node[draw, fill =white,minimum size= 3pt ]{};
	\draw (1cm, 0.16cm) node[draw, fill =white,minimum size= 3pt ]{};
 	\draw (1cm, -0.16cm) node[draw, fill =white,minimum size= 3pt ]{};

	\draw (0.5cm, 0.23cm) node[draw, fill =white,minimum size= 3pt ]{};	
	\draw (0.5cm, - 0.23cm) node[draw, fill =white,minimum size= 3pt ]{};	
	\draw (1.5cm, 0.23cm) node[draw, fill =white,minimum size= 3pt ]{};	
	\draw (1.5cm, -0.23cm) node[draw, fill =white,minimum size= 3pt ]{};	
	\draw (0.5cm, 0.115cm) node[draw, fill =white,minimum size= 3pt ]{};	
	\draw (0.5cm, - 0.115cm) node[draw, fill =white,minimum size= 3pt ]{};	
	\draw (1.5cm, 0.115cm) node[draw, fill =white,minimum size= 3pt ]{};	
	\draw (1.5cm, -0.115cm) node[draw, fill =white,minimum size=3pt ]{};

    \draw (2.5,-0.62) node[minimum size=3pt] (B1){};	
    \draw (3,-0.91) node[minimum size=3pt] (B2){};	
    \draw (3.5,-1.06) node[minimum size=3pt] (B3){};	
    \draw (4,-1.1) node[minimum size=3pt] (B4){};	
    \draw (4.5,-1.06) node[minimum size=3pt] (B5){};	
    \draw (5,-0.91) node[minimum size=3pt] (B6){};	
    \draw (5.5,-0.62) node[minimum size=3pt] (B7){};	
  
    \draw (B) to [out=-90,in=180] (B1){};
    \draw (B1) to [out=-85,in=195] (B2){};    
    \draw (B2) to [out=-75,in=215] (B3){};    
    \draw (B3) to [out=-60,in=245] (B4){};        
    \draw (B4) to [out=-60,in=245] (B5){};      
    \draw (B5) to [out=-50,in=265] (B6){};     
    \draw (B6) to [out=-40,in=275] (B7){};     
    \draw (B7) to [out=-0,in=270] (D){};

    \draw (B) to [out=-25,in=90] (B1){};
    \draw (B1) to [out=10,in=110] (B2){};    
    \draw (B2) to [out=40,in=110] (B3){};    
    \draw (B3) to [out=50,in=120] (B4){};        
    \draw (B4) to [out=60,in=135] (B5){};      
    \draw (B5) to [out=70,in=145 ] (B6){};     
    \draw (B6) to [out=80,in=170] (B7){};     
    \draw (B7) to [out=90,in=205] (D){};         

\end{tikzpicture}
\end{center}

%
	\caption{A graph for which no algorithm in $\Lazies$ is an optimal query policy. All arcs labeled by $A$ and $B$ include multi-edge structures shown in Figures~\ref{fig:thick} and \ref{fig:string}, respectively. For clarity, we include two examples of these structures on $sa$ and $at$ in this figure.}
	\label{fig:full}
\end{figure}

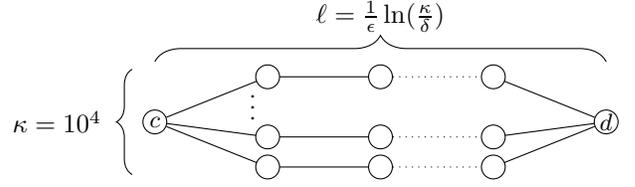
\begin{figure}[t]
	\vspace{-2.5mm}
	\begin{center}

\begin{tikzpicture}[node distance   = 1.5 cm]
	\usetikzlibrary{decorations.pathreplacing}
    \tikzstyle{every node}=[draw,circle,fill=white,minimum size=9pt,
                            inner sep=0pt]

    \draw (0,0) node (A) {\small $c$}
    --  ++(1.5cm, -0.2cm) node (B) {}
    --  ++(1.5cm, 0) node (C) {};   
    \draw (4.5cm,-0.2cm) node (D) {}
    --  ++(1.5cm,+0.2cm) node (E) {\small $d$};
     \draw [dotted] (C) to (D);   
      
    \draw (1.3cm, 0.3cm) node[draw = none, fill = none]{\small $\vdots$};      
      
     \draw (A) -- ++(1.5cm,0.6) node (B1) {}
     --  ++(1.5cm,0) node (C1) {}; 
    \draw (4.5cm,0.6) node (D1) {}  --  (E); 
     \draw [dotted] (C1) to (D1);         
       
    \draw (A) -- ++(1.5cm,-0.6) node (B2) {}
     --  ++(1.5cm,0) node (C2) {}; 
    \draw (4.5cm,-0.6) node (D2) {}  --  (E); 
     \draw [dotted] (C2) to (D2);        
      
    \draw [decorate,decoration={brace,amplitude=10pt}](0,0.8 cm) -- (6cm,0.8 cm) node [draw = none, fill = none, midway,yshift=0.6cm]{$\ell = \frac{1}{\epsilon} \ln(\frac \kappa \delta)$};
    \draw [decorate,decoration={brace,amplitude=6pt}](-0.3,-0.7 cm) -- (-0.3cm,0.7 cm) node [draw = none,fill = none,midway,xshift=-1cm]{$\kappa = 10^4$};
\end{tikzpicture}
\end{center}
    
	\caption{Structure $A$ used on arcs $sa$, $ab$, and $bt$ in Figure~\ref{fig:full}.
	We refer to one path connecting $c$ and $d$ as a ``string''.}
	\label{fig:thick}
\end{figure}

\begin{figure}[t]
	\begin{center}
\begin{tikzpicture}[node distance   = 1.5 cm]
	\usetikzlibrary{decorations.pathreplacing}
    \tikzstyle{every node}=[draw,circle,fill=white,minimum size=10pt,
                            inner sep=0pt]

    \draw (0,0) node (A) {\small $c$};
    \draw (1.5cm,0) node (B) {};
    \draw (3cm,0) node (C) {};
    \draw (4.5cm,0) node (D) {};
    \draw (6cm,0) node (E) {\small $d$};
    
     \draw [dotted] (C) to (D);

    \draw (A) to [out=45,in=135] (B);    
    \draw (B) to [out=45,in=135] (C);     
    \draw (D) to [out=45,in=135] (E);          
    \draw (A) to [out=-45,in=-135] (B);    
    \draw (B) to [out=-45,in=-135] (C);     
    \draw (D) to [out=-45,in=-135] (E);
    
    \draw [decorate,decoration={brace,amplitude=10pt}]
(0,0.5 cm) -- (6cm,0.5 cm) node [draw = none, fill = none, midway,yshift=0.55cm]{$\ell' = 3 \ell$};

    \draw [decorate,decoration={brace,amplitude=6pt}]
(-0.3,-0.5 cm) -- (-0.3cm,0.5 cm) node [draw = none,fill = none,midway,xshift=-1cm]{$\kappa' = 2$};
\end{tikzpicture}
\end{center}

	\caption{Structure $B$ used on arcs $sb$ and $at$ in Figure~\ref{fig:full}.}
	\label{fig:string}
\end{figure}
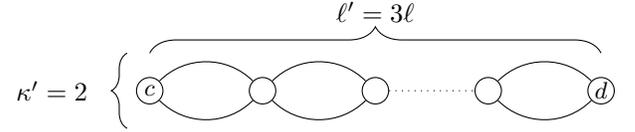

We compare the cost of any $\Lazy \in \Lazies$ (for an arbitrary edge selector) to that of an algorithm \alg defined as follows. \alg first queries all the edges in 
the multi-edge structures $B$ on arcs $sb$ and $at$. There are two cases:

\begin{enumerate}
\item A path exists in both of the structures $sb$ and $at$, or in neither one: In this case, \alg calls $\Lazy$ on the original graph.
\item There is a path in exactly one of the $sb$ or $at$ structures: Without loss of generality (by symmetry) assume that $at$ has a path. Then, \alg queries the edges in structure $A$ on $sa$, $ab$ and $bt$ in order, until it verifies that at least one of these structures does not have a path or all do. Then, it returns the shortest $\st$ path on the edges whose existence has been verified by the queries, or certifies that no $\st$ path exists.
\end{enumerate}

It is not hard to see that \alg demonstrates the required guarantees for a correct algorithm, i.e., upon its termination it correctly certifies that there is no $\st$ path or returns the shortest $\st$ path in the realized graph.

Let us provide an overview of why \alg queries fewer edges than any \Lazy algorithm in expectation.
The structures~$A$ and $B$ are designed so that structure $A$ requires more queries than structure $B$. Additionally, structure $A$ almost certainly fails to have a path, while structure $B$ has a path with a probability close to $\frac12$. Note that such a graph almost certainly does not have a path, so a large fraction of $\E[\mathrm{cost}(\alg(G, \O_{G'}))]$ comes from the effort required to \emph{invalidate} possible $\st$ paths.

In the first case of \alg (a path exists in both $ab$ and $at$, or in neither one), it queries more edges than \Lazy. However, we argue that \alg uses much fewer queries in its second case. The probability of existence of a path in structure $B$ is chosen so that the second case happens with significant probability (almost $\frac 12$), in which case the overall savings in the analysis of the second case bring down the total expected cost of $\alg$ compared to $\Lazy$.

In slightly more detail, the crux of the proof is the case where~$sb$ does not have a path and $at$ has a path (an example of the second case of \alg). To invalidate all possible $\st$ paths, it suffices to certify that structure $A$ on $sa$ does not have a path. Therefore,  $\alg$ terminates after querying only one $A$ structure, with high probability, in addition to querying two $B$ structures on $sb$ and $at$.
On the other hand, \Lazy does not know which one of $sb$ or $at$ has a path, so with probability at least $\frac 12$ it first queries some $A$ structure other than $sa$, in which case it has to also query and verify that no path exists in $sa$. Therefore, \Lazy has to query~$1.5$ $A$ structures in expectation.
We design structures~$A$ and~$B$ so that half the cost of checking an additional~$A$ structure is much larger than the initial cost that \alg invests in querying edges in two $B$ structures.

The next theorem and its proof formalize the foregoing discussion. 

\begin{theorem}
\label{thm:rand}
There is a graph $G =(V, E)$ and $p\in(0,1)$ for which the optimal query policy is not in $\Lazies$.
\end{theorem}

\begin{proof}
Consider the graph in Figure~\ref{fig:full}. Let 
\begin{align*}
&\kappa = 10^4,\\ 
&\kappa' = 2,\\ 
&\epsilon = 10^{-2},\\ 
&\delta = 10^{-3},\\  
&\ell = \frac{1}{\epsilon} \ln\left(\frac \kappa \delta\right),\\  
&\ell' = 3 \ell
\end{align*}
for the structures in Figures~\ref{fig:thick} and \ref{fig:string}. Let $p = 1 - \epsilon$ be the probability of existence of any one edge in these structures.  

In the following claim,
we show that the structure in Figure~\ref{fig:thick} almost certainly does not have a path, but one has to query many edges to verify that this is indeed the case.

\begin{claim}
\label{cl:A}
With probability at least $1-\delta$, there is no $c \dash d$ path in the structure shown in Figure~\ref{fig:thick}. Conditioned on the event that no $c\dash d$ path exists, any correct querying policy has to query at least $10^6 - 2$ edges in expectation to certify that no path exists. 
\end{claim}
\begin{claimproof}
The probability of a path existing in this structure is at most
\[
\kappa (1- \epsilon)^\ell \leq \kappa e^{-\epsilon \ell} = \delta.
\]
Let $\mathcal{E}$ be the event that no path exists in the structure. Given~$\mathcal{E}$, each of the $\kappa$ strings of length $\ell$ have to be invalidated. Consider the expected number of queries needed to invalidate a single string, conditioned on $\mathcal{E}$. For $i=1,\ldots,\ell$, let $\mathcal{F}_i$ be the event that the first $i-1$ queried edges in the string exist in $G'$, and the $i^{th}$ edge does not. Clearly the events $\mathcal{F}_i$ and $\mathcal{E}$ are positively correlated, that is, for all $i=1,\ldots,\ell$,
$\Pr[\mathcal{F}_i\ |\ \mathcal{E}]\geq \Pr[\mathcal{F}_i]$.
Therefore, conditioned on $\mathcal{E}$, the expected number of queries on a string is
\begin{align*}
\sum_{i=1}^\ell \Pr[\mathcal{F}_i|\mathcal{E}]\cdot i & \geq \sum_{i=1}^\ell \Pr[\mathcal{F}_i]\cdot i\\
&=\sum_{i=1}^{\ell} (1-\epsilon)^{i-1} \epsilon i\\
  &= \frac{1 - (1-\epsilon)^{\ell} - \ell (1-\epsilon)^{\ell} \epsilon}{\epsilon}\\
& \geq \frac 1\epsilon - 2 \cdot 10^{-4}.
\end{align*}
Using the linearity of expectation and summing over all~$\kappa$ disjoint strings that have to be invalidated, the expected number of queries needed to invalidate the structure is at least
$$\kappa \left( \frac 1\epsilon - 2 \cdot 10^{-4} \right) = 10^6 - 2.$$
\end{claimproof}

\vspace{2mm}

In the next claim, we show that the structure in Figure~\ref{fig:string}, though narrower and longer than the structure in Figure~\ref{fig:thick}, has a path with higher probability. 

\begin{claim}
\label{cl:B}
With probability $0.616 \pm 10^{-3}$ there is a path in the structure shown in Figure~\ref{fig:string}. Moreover, the expected number of queries needed to find a path or certify that none exists is at most $10^{4}$.
\end{claim}

\begin{claimproof}
The probability of a path existing in this structure is exactly 
\[
(1- \epsilon^{\kappa'})^{\ell'} = (1 - 0.01^2)^{\frac{3}{0.01} \ln(10^7)} = 0.616 \pm 10^{-3}.
\]
Moreover, since the structure has $\kappa'\ell'$ edges overall, the expected number of queries is also bounded by $\kappa'\ell' \leq 10^{4}$.
\end{claimproof}

\vspace{1.5mm}

We now turn to comparing the performance of \alg with that of \Lazy. First, note that \alg queries at most $2 \cdot 10^4$ edges for verifying arcs $sb$ and $at$ at the beginning, whereas \Lazy may not query those edges. Consider the following cases:
\begin{enumerate}
	\item There is a path in at least one of the $A$ structures $sa$, $ab$, or $bt$.
	\item There is no path in the $A$ structures $sa$, $ab$ and $bt$, and exactly one of the $B$ structures $sb$ or $at$ has a path.
	\item Cases 1 and 2 do not hold.
\end{enumerate}

Consider Case 1. By Claim~\ref{cl:A}, this is a rare event that happens with probability at most $3\delta$. Conditioned on this event, \alg verifies at most three $A$ structures in addition to arcs~$sb$ and $at$, with overall number of edges $3\kappa \ell$. Taking the probability of this event into account, \alg issues at most $$3\delta \cdot 3 \kappa \ell \leq 1.46 \cdot 10^5$$ more queries in expectation (in addition to the $B$ structures which we will account for separately).

Consider Case 2. By Claims~\ref{cl:A} and \ref{cl:B}, this event happens with probability at least $$2(1- 3\delta) \cdot(0.616 \pm10^3)\cdot (1 - 0.616\mp10^3) \geq 0.471.$$ 
Conditioned on this event, \alg invalidates one $A$ structure in addition to verifying arcs $sb$ and $at$. This is because $\alg$ only needs to query and invalidate the $A$ structure that is parallel to the non-valid $B$ structure on $sb$ or $at$. For example, when arc $at$ has a path and $sb$ does not, it suffices to invalidate structure~$sa$ to certify that no $\st$ path exists in Figure~\ref{fig:full}.

On the other hand, conditioned on the event that exactly one of the structures $sb$ or $at$ has a path, \Lazy has to invalidate $1.5$ $A$ structures in expectation. Indeed, initially it must query edges on the shortest path, and they are all in~$A$ structures. It can only query $sb$ or $at$ after an $A$ structure has been invalidated, but, at that point, with probability $1/2$ there might still be a path using another $A$ structure, chained with the valid $B$ structure. 

In Case 3, \alg verifies at most $2$ more $B$ structures than \Lazy (this is Case $1$ of \alg).

To summarize,
\vspace{-1.0mm}
\begin{align*}
&\E_{G'\sim G}[\cost (\alg(G, \O_{G'}))]\\
&\quad\quad \leq 
\E_{G'\sim G}\left[ \cost(\Lazy(G, \O_{G'})) \right]  \\
&\quad\quad \quad\quad + \underbrace{2 \cdot 10^4}_{\text{verifying B structures}}
+\underbrace{1.46 \cdot 10^5}_{\text{Case 1}} - \underbrace{2.35 \cdot 10^5}_{\text{Case 2}} \\
&\quad\quad < 
\E_{G'\sim G}\left[ \cost(\Lazy(G, \O_{G'})) \right]. 
\end{align*}
%
\end{proof}

It may be instructive to understand why the Example of Figure~\ref{fig:full} does not contradict Theorem~\ref{thm:det}. Consider the potentially problematic Case 2 of the proof of Theorem~\ref{thm:rand}, where, say, the arc $sb$ is collision-free, and the arc $at$ is not; moreover, the three $A$ structures are in collision. Then $Q^*=Q^*_n$ (as defined in the proof of Theorem~\ref{thm:det}) would be a set of edges that invalidates $at$ and $bt$. In the deterministic setting, \Lazy with an omniscient edge selector could start by invalidating $bt$, then proceed to $at$.

\subsection{Query Complexity Bounds}

Theorem~\ref{thm:rand} implies that algorithms in $\Lazies$ may be suboptimal in the probabilistic setting. Nevertheless, it may still be possible to give satisfying worst-case guarantees with respect to the performance of algorithms in this class. This is exactly what we do next.

Specifically, we first show that any algorithm in  $\Lazies$ (with an edge selector satisfying a certain property) uses $O(n/p)$ queries, where $n=|V|$, with high probability. We then show that there is a graph where no correct path-finding algorithm terminates within $\omega(n/p)$ queries. Taken together, these results show that no other algorithm can hope to do significantly better than algorithms in $\Lazies$ over all underlying graphs.


In our upper bound, we focus on edge selectors that choose an unqueried edge between two connected components formed by the validated queried edges.

\begin{definition}
An edge selector $f: \P \times 2^E \times 2^E$ is \emph{connective} if for any $P\in\P$ and edge sets $Q_y$ and $Q_n$, $f(P,Q_y, Q_n)$ returns an edge $e\in P\setminus (Q_y \cup Q_n)$ that connects two connected components of the subgraph $(V, Q_y)$.  
\end{definition}

It is not hard to see that the bisection edge selector (defined in Section~\ref{sec:model}) is not a connective edge selector. On the other hand, both forward and backward edge selectors are connective. 

Let us provide an overview of why the forward edge selector --- used with a \Lazy algorithm that breaks ties in favor of paths with more verified edges --- is connective (the same argument applies to the backward edge selector, switching the roles of $s$ and $t$). Note that at any time the set of verified edges forms a connected component around vertex $s$. Moreover, by the same reasoning behind Dijkstra~\cite{D59}, if a vertex $v$ is in that connected component, the shortest $s$-$v$ path in $G'$ has been found. Now, refer to Figure~\ref{fig:forward}, and consider the path $P_1,v,v',R$, for two vertices $v$ and $v'$ that are already reachable from $s$ (i.e., $P_1\subseteq Q_y$ and $P_2\subseteq Q_y$), and $R\subseteq E\setminus Q$. Then \Lazy would prefer the path $P_2,R$, because $|P_2|\leq |P_1|+1$ (as it is the shortest path to $v'$), and $P_2$ is fully verified. We conclude that \Lazy with the forward edge selector never queries an edge within a connected component.


\begin{figure}[t]
\begin{center}
\begin{tikzpicture}[scale = 1.2]
	\usetikzlibrary{decorations.pathreplacing,decorations.pathmorphing, snakes }
    \tikzstyle{every node}=[draw,circle,fill=white,minimum size=12pt,
                            inner sep=0pt] 
    \draw (0,0) node (S) {{\small $s$}};
    \draw (2cm, 1.5cm) node (VV) {{\small $v$}};
    \draw (1.5cm,0cm) node (V) {{\small $v'$}};
    \draw (4cm,0cm) node (T) {{\small $t$}};

    \draw[decorate,decoration=snake,thick](S) to [out=80,in=190] (VV){};
	\draw[decorate,decoration=snake] (V) to [out=0,in=180] (T){};
 	\draw[decorate,decoration=snake,thick] (S) to (V){};
    \draw [dashed] (V) to (VV){};

  \draw (0.7cm, 1.5cm) node[draw = none, fill = none]{$P_1$};
  \draw (0.7cm, -0.25cm) node[draw = none, fill = none]{$P_2$};
  \draw (2.8cm, -0.25cm) node[draw = none, fill = none]{$R$};

\end{tikzpicture}
\end{center}
\caption{\Lazy with the forward edge selector does not query an edge between two vertices in the same connected component.} \label{fig:forward}
\end{figure}
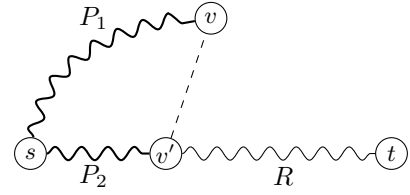

We now turn to deriving a rigorous upper bound on the number of edges queried by any \Lazy algorithm with a connective edge selector. Although the proof seems straightforward in retrospect, it is actually rather tricky. In terms of implications, we view this theorem as our main result. 

\begin{theorem}
\label{thm:upper}
For any $\delta>0$, $p\in(0,1)$, graph $G$ with $n$ vertices, and a connective edge selector $f$, with probability at least $1-\delta$, 
\[ \cost(\Lazy_f(G, \O_{G'})) \in O\left(\frac{n + \ln(1/\delta)}{p} \right).
\]
\end{theorem}
\begin{proof}
Let  
$$m = \frac{1}{p} \max\left\{ 2n,  8 \ln\left(\frac{1}{\delta}\right) \right\},$$ 
and let $\calE$ be the event that $\cost(\Lazy_f(G, \O_{G'})) > m$. 

Consider  $m$ independent  Bernoulli random variables each with parameter $p$, $\vec X=(X_1, \dots, X_m)$. Let $X_i = 1$ correspond to the event where the $i^{th}$ edge queried by $\Lazy_f$ is in $E'$, and $X_i = 0$ otherwise. Intuitively, we think of $X_1,\ldots,X_m$ as flipping coins with bias $p$ \emph{in advance} to decide the answers to the queries issued by \Lazy. We can do this because the probability that the answer to a query is \textsc{Yes} is independent of which edge is queried.  

Formally, let $\Pr_{(\vec X, G')}[\cdot]$ correspond to taking probability over a random process that generates $G'$ by first instantiating the Bernoulli random variables $X_1,\ldots,X_m$, then determining the corresponding sets of edges $Q_y$ and $Q_n$ that are validated and invalidated by \Lazy, respectively.
For any edge  $e\in Q_y$ or $e\in Q_n$, set $e$ to belong to, or not belong to $E'$, respectively. For any edge $e\in E\setminus (Q_y \cup Q_n)$, set $e$ to  belong to $E'$ with probability $p$, independently.
Note that in this process each edge belongs to $E'$ with probability exactly $p$. So,  $\Pr_{G'\sim G}[\calE] = \Pr_{(\vec X, G')}[\calE].$ 
Using conditional probability, we have
\begin{align}
\Pr_{G'}\left[\calE \right]&{=}\!\! \Pr_{( \vec X, G')}\!\left[ \calE ~\left|~ \sum_{i=1}^m{X_i}<n\right.\right]\!  \Pr_{( \vec X, G')}\!\!\left[ \sum_{i=1}^m X_i < n \right] \label{eq:line1}\\
&\!\!{+}\Pr_{( \vec X, G')}\! \left[ \calE ~\left|~ \sum_{i=1}^m X_i \geq n \right.\right]\!  \Pr_{( \vec X, G')}\!\!\left[ \sum_{i=1}^m X_i \geq n \right] \label{eq:line2}
\end{align}

In the following, we analyze  terms \eqref{eq:line1} and \eqref{eq:line2}, separately. For the first term, we have
\begin{align*}
(\ref{eq:line1}) ~ & \leq  \Pr_{( \vec X, G')}\left[ \sum_{i=1}^m X_i < n \right]  =  \Pr_{\vec X}\left[ \sum_{i=1}^m X_i < n \right] \leq \delta,
\end{align*}
where the last inequality is a direct consequence of the Chernoff bound:
\begin{align*}
\Pr_{\vec X}\left[ \sum_{i=1}^m X_i < n \right] &\leq \Pr_{\vec X}\left[ \sum_{i=1}^m X_i  < \frac{mp}{2} \right] \leq \exp\left(-\frac{mp}{8}\right). 
\end{align*}

Next, we argue that the second term, in Equation \eqref{eq:line2}, is zero. Specifically, we show that 
\begin{equation}
\label{eq:zero}
\Pr_{( \vec X, G')}\! \left[ \calE~ \left|~ \sum_{i=1}^m X_i \geq n \right.\right] = 0.
\end{equation}
Indeed recall that $Q_y$ denotes the set of edges validated by $\Lazy_f$ during the first $m$ queries, corresponding to $X_1, \dots, X_m$. Note that if $\Lazy_f$ does not terminate within the first $m$ queries, then $s$ and $t$ belong to two different connected components of the graph $H = (V, Q_y)$. By the connectivity property of $f$, at any time $\Lazy_f$ only queries an edge that is between two connected components of the validated edges at that time. So, every time $\Lazy_f$ encounters a queried edge that is realized (that is,~$Q_y$ grows), the number of connected components in $H$ decreases. Therefore, after encountering~$n$ verified edges, i.e., $\sum X_i \geq n$, there is an $\st$ path in $H$. This establishes Equation~\eqref{eq:zero}.

Combining terms \eqref{eq:line1} and \eqref{eq:line2}, we have that $\Pr_{G'}\left[ \calE \right] \leq \delta.$
\end{proof}

In the next theorem, we provide a matching lower bound for the number of queries that \emph{any correct path finding algorithm} requires. 

\begin{theorem}
\label{thm:lower}
For all $p\in(0,1)$ and $n>15$, there exists a graph $G$ with $n$ vertices such that for any correct path-finding algorithm $\alg$,
\[ \Pr_{G'} \left[   \cost(\alg(G, \O_{G'})) \leq \frac{n-1}{2p} \right] \leq 0.1.
\]
\end{theorem}
\begin{proof}
Let $m = \left\lfloor (n-1)/2p\right\rfloor$.
Consider the following graph $G = (V, E)$: Let $V$ be the sequence of vertices $s= v_1, v_2, \dots, v_n = t$. For each $i=1,\ldots,n-1$, let there be $m+1$ parallel edges between $v_i$ and $v_{i+1}$. 

Since there are $m+1$ parallel edges between any two vertices, $\alg$ cannot certify that no $\st$ path exists in $G'$ with only $m$ queries. So, if $\alg$ terminates with at most $m$ queries, it is because it has found an $\st$ path. To find an $\st$ path, $\alg$ must have encountered at least $n-1$ realized edges between the $m$ queries it has made. Therefore, using the same Bernoulli random variables as in the proof of Theorem~\ref{thm:upper}, the Chernoff bound, and the fact that $n>15$, we have
\small
\begin{align*}
 \Pr_{G'} \left[   \cost(\alg(G, \O_{G'})) \leq m \right] & \leq  \Pr_{X_1, \dots, X_m} \left[  \sum_{i=1}^m {X_i} \geq n-1 \right] \\
  &\leq  \Pr_{X_1, \dots, X_m} \left[  \sum_{i=1}^m {X_i} \geq 2m p \right] \\
  & \leq \exp( - mp/3) \leq 0.1.
\end{align*}
\normalsize
\end{proof}

\section{Discussion}
\label{sec:disc}

We wrap up by briefly discussing some pertinent issues. 

\medskip
\noindent\textbf{Multigraphs are mostly for ease of exposition.} Recall that the graph $G$ can have multiple edges between two vertices. As we mentioned in Section~\ref{sec:model}, this assumption is ``mostly'' for ease of exposition. Clearly, our positive results, Theorems~\ref{thm:det} and \ref{thm:upper}, hold even for simple graphs. Our first negative result, Theorem~\ref{thm:rand}, holds for simple graphs but the construction becomes (even) more unwieldy. The only exception is Theorem~\ref{thm:lower}: we were unable to establish it for simple graphs (it is easy to prove a lower bound of $\Omega(n/\ln n)$ for constant~$p$, though). 

\medskip
\noindent\textbf{Is the connectivity assumption needed?} 
Theorem~\ref{thm:upper} holds for \Lazy with a connective edge selector. Even though the current proof strongly relies on the connectivity assumption, we have not found examples of edge selectors that violate the theorem's conclusion (in particular, we have not been able to find a bad example for the bisection edge selector). We therefore conjecture that the $O(n/p)$ query complexity upper bound holds for \emph{any} edge selector, as long as \Lazy breaks ties in favor of paths with more verified edges (otherwise it is easy to construct bad examples). Despite significant effort on our part, this conjecture remains open. 

\medskip
\noindent\textbf{Computation of the optimal policy in the probabilistic setting.} 
Our results suggest that \Lazy is an excellent proxy for the optimal policy in the probabilistic setting, in that with, say, the forward edge selector, it is computationally efficient and provides satisfying guarantees. This is backed up by the empirical evaluation presented by Dellin and Srinivasa~\cite{DS16}.
One may ask, though, whether the optimal policy itself can be computed. The answer is that this seems to be an extremely hard problem. The most direct representation of the problem is via a Markov decision process (MDP), where there is a state for every possible choice of $Q_y$ and $Q_n$, the action space is edges in $E\setminus Q$, and the transitions and rewards are defined in the obvious way. Although an optimal policy in an MDP can be computed in polynomial time in its representation, the difficulty is that the size of the state space is exponential in $|E|$. That said, heuristics for (exactly or approximately) computing the optimal policy in the probabilistic setting have the potential to provide a practical alternative to \Lazy.

\bibliographystyle{aaai}
\bibliography{bibliography}

\begin{thebibliography}{}

\bibitem[\protect\citeauthoryear{Bohlin and Kavraki}{2000}]{BK00}
Bohlin, R., and Kavraki, L.~E.
\newblock 2000.
\newblock Path planning using lazy {PRM}.
\newblock In {\em {IEEE} Int. Conf. on Robotics and Automation (ICRA)},
  521--528.

\bibitem[\protect\citeauthoryear{Choset \bgroup et al\mbox.\egroup
  }{2005}]{CBHKKLT05}
Choset, H.; Lynch, K.~M.; Hutchinson, S.; Kantor, G.; Burgard, W.; Kavraki,
  L.~E.; and Thrun, S.
\newblock 2005.
\newblock {\em Principles of Robot Motion: Theory, Algorithms, and
  Implementation}.
\newblock MIT Press.

\bibitem[\protect\citeauthoryear{Choudhury \bgroup et al\mbox.\egroup
  }{2017}]{CSCS17}
Choudhury, S.; Salzman, O.; Choudhury, S.; and Srinivasa, S.~S.
\newblock 2017.
\newblock Densification strategies for anytime motion planning over large dense
  roadmaps.
\newblock In {\em {IEEE} Int. Conf. on Robotics and Automation (ICRA)},
  3770--3777.

\bibitem[\protect\citeauthoryear{Choudhury, Dellin, and
  Srinivasa}{2016}]{CDS16}
Choudhury, S.; Dellin, C.~M.; and Srinivasa, S.~S.
\newblock 2016.
\newblock Pareto-optimal search over configuration space beliefs for anytime
  motion planning.
\newblock In {\em {IEEE/RSJ} Int. Conf. on Intelligent Robots and Systems
  (IROS)},  3742--3749.

\bibitem[\protect\citeauthoryear{Dellin and Srinivasa}{2016}]{DS16}
Dellin, C.~M., and Srinivasa, S.~S.
\newblock 2016.
\newblock A unifying formalism for shortest path problems with expensive edge
  evaluations via lazy best-first search over paths with edge selectors.
\newblock In {\em Int. Conf. on Automated Planning and Scheduling (ICAPS)},
  459--467.

\bibitem[\protect\citeauthoryear{Dijkstra}{1959}]{D59}
Dijkstra, E.~W.
\newblock 1959.
\newblock A note on two problems in connexion with graphs.
\newblock {\em Numerische Mathematik} 1(1):269--271.

\bibitem[\protect\citeauthoryear{Hart, Nilsson, and Raphael}{1968}]{HNR68}
Hart, P.~E.; Nilsson, N.~J.; and Raphael, B.
\newblock 1968.
\newblock A formal basis for the heuristic determination of minimum cost paths.
\newblock {\em {IEEE} Transactions on Systems, Science, and Cybernetics}
  SSC-4(2):100--107.

\bibitem[\protect\citeauthoryear{Hauser}{2015}]{K15}
Hauser, K.
\newblock 2015.
\newblock Lazy collision checking in asymptotically-optimal motion planning.
\newblock In {\em {IEEE} Int. Conf. on Robotics and Automation (ICRA)},
  2951--2957.

\bibitem[\protect\citeauthoryear{Hopcroft, Schwartz, and Sharir}{1984}]{HSS84}
Hopcroft, J.~E.; Schwartz, J.~T.; and Sharir, M.
\newblock 1984.
\newblock On the complexity of motion planning for multiple independent
  objects; {PSPACE}-hardness of the ``warehouseman's problem''.
\newblock {\em I. J. Robotics Res.} 3(4):76--88.

\bibitem[\protect\citeauthoryear{Hsu, Latombe, and Motwani}{1999}]{HLM99}
Hsu, D.; Latombe, J.; and Motwani, R.
\newblock 1999.
\newblock Path planning in expansive configuration spaces.
\newblock {\em Int. J. Comput. Geometry Appl.} 9(4--5):495--512.

\bibitem[\protect\citeauthoryear{Karaman and Frazzoli}{2011}]{KF11}
Karaman, S., and Frazzoli, E.
\newblock 2011.
\newblock Sampling-based algorithms for optimal motion planning.
\newblock {\em I. J. Robotics Res.} 30(7):846--894.

\bibitem[\protect\citeauthoryear{Kavraki \bgroup et al\mbox.\egroup
  }{1996}]{KSLO96}
Kavraki, L.~E.; Svestka, P.; Latombe, J.; and Overmars, M.~H.
\newblock 1996.
\newblock Probabilistic roadmaps for path planning in high-dimensional
  configuration spaces.
\newblock {\em {IEEE} Trans. Robotics and Automation} 12(4):566--580.

\bibitem[\protect\citeauthoryear{LaValle and Kuffner}{1999}]{LK99}
LaValle, S.~M., and Kuffner, J.~J.
\newblock 1999.
\newblock Randomized kinodynamic planning.
\newblock In {\em {IEEE} Int. Conf. on Robotics and Automation (ICRA)},
  473--479.

\bibitem[\protect\citeauthoryear{LaValle}{2006}]{L06}
LaValle, S.~M.
\newblock 2006.
\newblock {\em Planning Algorithms}.
\newblock Cambridge University Press.

\bibitem[\protect\citeauthoryear{Lozano-Perez}{1983}]{L83}
Lozano-Perez, T.
\newblock 1983.
\newblock Spatial planning: A configuration space approach.
\newblock {\em IEEE Transactions on Computers} C-32(2):108--120.

\bibitem[\protect\citeauthoryear{Salzman and Halperin}{2015}]{SH15}
Salzman, O., and Halperin, D.
\newblock 2015.
\newblock Asymptotically-optimal motion planning using lower bounds on cost.
\newblock In {\em {IEEE} Int. Conf. on Robotics and Automation (ICRA)},
  4167--4172.

\bibitem[\protect\citeauthoryear{Sharir}{2004}]{S04}
Sharir, M.
\newblock 2004.
\newblock Algorithmic motion planning.
\newblock In {\em Handbook of Discrete and Computational Geometry, Second
  Edition.}
\newblock  1037--1064.

\end{thebibliography}

\end{document}